\documentclass{article}

     \PassOptionsToPackage{numbers, compress}{natbib}

    \usepackage[preprint]{neurips_2021}

\usepackage{amsmath,amsfonts,bm}

\def\eqref#1{equation~\ref{#1}}

\def\1{\bm{1}}

\DeclareMathAlphabet{\mathsfit}{\encodingdefault}{\sfdefault}{m}{sl}
\SetMathAlphabet{\mathsfit}{bold}{\encodingdefault}{\sfdefault}{bx}{n}

\usepackage{url}

\usepackage[utf8]{inputenc} 
\usepackage[T1]{fontenc}    
\usepackage{url}            
\usepackage{booktabs}       
\usepackage{amsfonts}       
\usepackage{nicefrac}       
\usepackage{microtype}      
\usepackage{xcolor}         
\usepackage{amsthm}
\usepackage[numbers]{natbib}
\usepackage{graphicx}
\usepackage{caption}
\usepackage{subcaption}

\newtheorem{theorem}{Theorem}

\def\mytitle{Certified Defense via Latent Space Randomized Smoothing with Orthogonal Encoders}
\title{\mytitle}

\title{Certified Defense via Latent Space Randomized Smoothing with Orthogonal Encoders}

\author{%
  Huimin Zeng\\
  Technical University of Munich \\
  \texttt{huimin.zeng@tum.de} \\
   \And
   Jiahao Su \\
   University of Maryland, College Park \\
   \texttt{jiahaosu@umd.edu} \\
   \AND
   Furong Huang \\
   University of Maryland, College Park \\
   \texttt{furongh@cs.umd.edu} \\
}

\begin{document}

\maketitle

\begin{abstract}
Randomized Smoothing (RS), being one of few provable defenses, has been showing great effectiveness and scalability in terms of defending against $\ell_2$-norm adversarial perturbations. However, the cost of MC sampling needed in RS for evaluation is high and computationally expensive. To address this issue, we investigate the possibility of performing randomized smoothing and establishing the robust certification in the latent space of a network, so that the overall dimensionality of tensors involved in computation could be drastically reduced. To this end, we propose Latent Space Randomized Smoothing. Another important aspect is that we use orthogonal modules, whose Lipschitz property is known for free by design, to propagate the certified radius estimated in the latent space back to the input space, providing valid certifiable regions for the test samples in the input space. Experiments on CIFAR10 and ImageNet show that our method achieves competitive certified robustness but with a significant improvement of efficiency during the test phase. 
\end{abstract}

\section{Introduction}
\label{sec:intro}
Deep neural networks (DNNs) show impressive performance in many tasks, like image recognition, language understanding and audio processing. However, it is also widely known that deep neural networks can be vulnerable to adversarially perturbed input examples~\citep{szegedy2013intriguing}. Therefore, it is important to have strong defenses against such adversarial examples, especially for security-critical scenarios. 
Empirical defenses have been proposed to train robust classifiers~\cite{AdversarialExamplesAreNotEasilyDetected:BypassingTenDetectionMethods, Kannan2018AdversarialLP, kurakin2016adversarial, shaham2018understanding}. Despite their successes in defending against certain attacks, there is no worst-case performance guarantees. Another line of work \cite{raghunathan2018certified, wong2017provable} is called certified defense, which provides provable robustness: bounded perturbation in the input will not cause a large change in the output. However, existing certified defense fails to scale to large datasets or apply to arbitrary model types.  

\emph{Input Space Randomized Smoothing (IS-RS)} \cite{cohen2019certified} provides certified robustness of a smoothed classifier, smoothed-out by \emph{input space Gaussian augmentations}. Compared with other certified defenses, {IS-RS} guarantees certified robustness scalable to large-scale datasets and large network architectures. However, in this work, we argue that the efficiency of IS-RS can be further improved. There are three sources for the inefficiencies in IS-RS: (1) The deep depth of the network that the Gaussian noise has to (forward/backward) propagate through. (2) The large number of Gaussian samples needed for an accurate enough empirical estimation of the expectation over the Gaussian distribution. (1) and (2) are the most important factors. In addition, there is another minor issue: (3) The high-dimensionality of the space in which we sample Gaussian noise. These three factors will negatively affect {the efficiency of} randomized smoothing, in terms of training, inference as well as the final robustness evaluation. It is true that training and inference of the classifier may not require a lot of Gaussian samples per example, but if we consider the entire training set and test set, the overall number of Gaussian noises needed is large. More importantly, evaluating the certification for each data point requires a large quantity of Gaussian noises to maintain the high confidence (usually over 100,000 for a single image). 
Therefore, it is crucial to increase the efficiency of IS-RS.

Motivated by observing the sources of inefficiencies of IS-RS, we propose \emph{latent space Randomized Smoothing (LS-RS)}, which achieves efficient certified robustness of a classifier smoothed by injecting Gaussian noise in the (compact) latent space. In order to define the latent space of a network more precisely, we split a normal neural network $f$ into two sub-networks, as shown in Figure~\ref{fig:split}: an \textbf{encoder} $f_e$ and a \textbf{classifier} $f_c$. The latent space is the space that links the output feature space of $f_e$ and the input feature space of $f_c$. Under this setting, the forward pass of any given input is computed with $f(\bm{x}) = f_c(f_e(\bm{x}))$ and the output of $f_e$, i.e. $\bm{z}$, will be smoothed with Gaussian noises to obtain the a smoothed function.

\begin{figure}[!htbp]
  \centering
  \includegraphics[width=0.8\linewidth]{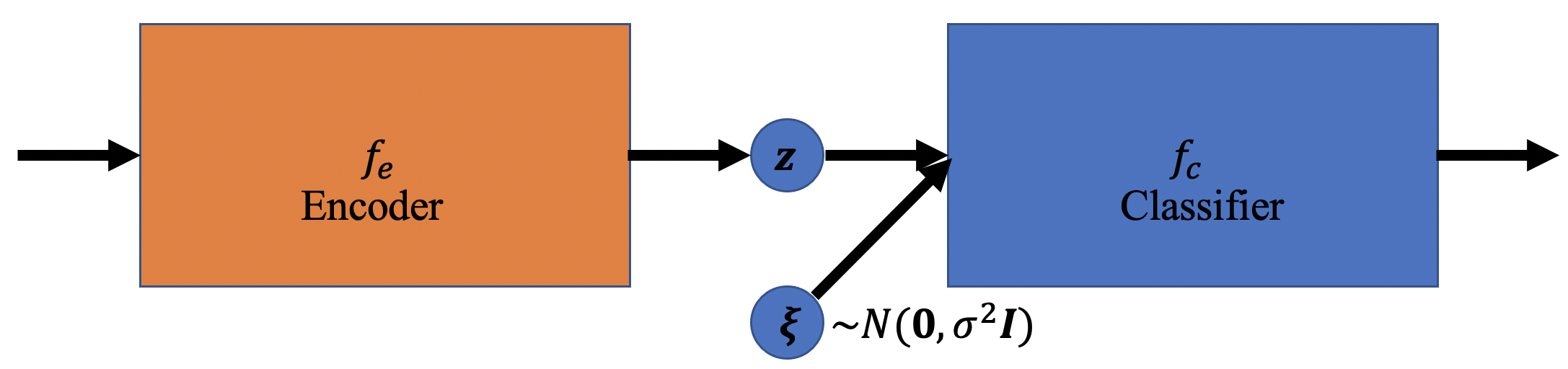} 
  \caption{Split a network into two parts.}
  \label{fig:split}
\end{figure}

LS-RS is computational more efficient than IS-RS due to the following reasons.
\textbf{(1)} LS-RS could be implemented at deeper layers, 
therefore the effective depth required for forward/backward propagation of the Gaussian augmented images could be significantly reduced. 
\textbf{(2)} LS-RS could be implemented on a compact latent space, with a representation size potentially significantly smaller than the size of the input layer, improving the inefficiency rooted from the high-dimensionality of the space to inject Gaussian noise. In fact, LS-RS provides us with great \textbf{flexibility} of choosing the latent space, via which the dimensionality of latent space could be easily controlled.


There are a few technical challenges to implement an LS-RS algorithm. \textbf{Challenge I:} A certified robustness guarantee against adversarial perturbations in the input space, for a LS-RS model, does not exist. The reason is that the Gaussian noises are sampled in the latent space instead of the input space. To this end, the robustness certification of randomized smoothing could only be established for $f_c$ instead of $f$, which is frustratingly useless, since in practice, the adversarial perturbations are usually created for corrupting the input space. 
\textbf{Challenge II:} Even if a certified radius in the latent space is obtained, reverting it back to the input space requires a tight characterization of the Lipschitz constant of the sub-network (encoder) $f_e$, mapping the data in the input space to that in the latent space. Computing the Lipschitz constant of a network is usually computationally difficult.
\textbf{Challenge III:} The Lipschitz constant of a linear layer is an upper bound of the singular values across all singular vectors (i.e., spectral components). An algorithm to compute the Lipschitz constant of the sub-network $h$ might require computation of the SVD of convolutional layers, which is too expensive and thus violating our purpose of improving the efficiency of IS-RS. In addition, more seriously, under bad condition numbers, i.e., the singular values vary drastically across different spectral components, the Lipschitz constant of the network might be too loose to obtain competitive certified radius in the input space.

In this paper, to solve the aforementioned challenges in implementing LS-RS, we propose a novel design of the encoder $f_e$ to be an \emph{orthogonal convolutional network}~\cite{jia2019orthogonal, wen2020towards}.
In other words, we design the encoder $f_e$ such that for any $\bm{x}$ and $\bm{x}'$, $\lVert \bm{f_e}(\bm{x})-\bm{f_e}(\bm{x}') \rVert_2= \lVert \bm{x} - \bm{x}'\rVert_2$ and $\lVert \bm{\nabla{f_e}}{\bm{x}} \rVert_2 = 1$. It is obvious that the singular values of all spectral components $\equiv 1$ by model design. With this guaranteed ``flattness'' of the spectral components for the sub-network, we obtain the Lipschitz constant for free without any additional computation and our characterization of the Lipschitz constant is tight enough for a competitive certified radius in the input space, as verified in our experiments.

We established an equivalency between IS-RS and LS-RS in terms of certified accuracy under our framework.



\paragraph{Summary of contributions:} 
\begin{enumerate}
    \item We introduce a novel latent space randomized smoothing by augmenting the latent feature representation with i.i.d. Gaussian noise instead of the input images. 
    
    \item Based on LS-RS, we  significantly increase the efficiency of the randomized smoothing framework from different aspects. For instance, the choice of the latent space is flexibility, which could further reduce the forward/backward complexity. On CIFAR10 dataset \cite{krizhevsky2009learning}, the average time used for certifying one single image could be reduced from around 12 seconds to around 8 seconds, a 33.3\% efficiency improvement, with only 5.93\% degradation in performance for $\sigma = 0.25$, 5.50\% for $\sigma = 0.50$ and 1.50\% for $\sigma = 1.0$. Similarly, on ImageNet \cite{russakovsky2015imagenet}, we can also observe a significant improvement on test efficiency without sacrificing much accuracy.

    \item Finally, we point out a technical contribution of adopting orthogonal convolutinoal layers and the norm-preserving GroupSort non-linear activation to build the sub-network $f_e$. Without exhaustively and expensively computing the Lipschitz of the non-convex encoder, we are able to easily find equivalency between the certified radius of the input space and that of the latent space.
    
\end{enumerate}

\section{Methodology}
\label{sec:alg}
We will start by reviewing the randomized smoothing framework. Let us consider a classification model $f: \bm{\mathcal{X}} \to \mathcal{Y}$, mapping examples
in an input space $\bm{x} \in \bm{\mathcal{X}}$ to a label $y \in \bm{\mathcal{Y}}$ in the label space.
Then there exists a robustness guarantee for a ``smoothed'' version $g$ of the base classifier $f$. Formally, the smoothed classifier $g$ is defined as:
\begin{equation}
    g(f, \bm{x}, \sigma^2) = \underset{\hat{y} \in \mathcal{Y}}{\mathrm{argmax}} \mathbb{P}(f(\bm{x} + \bm{\xi}) = \hat{y}), \quad \mathrm{where} \quad \bm{\xi} \sim \mathcal{N}(\bm{0}, \sigma^2 \bm{I}).
    \label{eq:smoothed_classifier}
\end{equation}

In other words, given a test input $\bm{x}$, a smoothed classifier $g$ will augment it with isotropic Gaussian noises (parameterized by $\sigma$), and predict a label that the majority of the augmented images output after propagating through the base classifier $f$.
If we denote the  class output by the majority of the Gaussian augmented images as
$$y_A = g(f, \bm{x}, \sigma^2) = \underset{y_A \in \mathcal{Y}}{\mathrm{argmax}} \mathbb{P}(f(\bm{x} + \bm{\xi}) = y_A)$$ and any other ``runner-up'' class as $$y_B = \underset{y_B \neq y_A}{\mathrm{argmax}} \mathbb{P}(f(\bm{x} + \bm{\xi}) = y_B),$$ then the result of the smoothed classifier under any perturbation of the input within a radius $R$ will be robust
\begin{equation}
    \begin{split}
        g(\bm{x} + \bm{\delta
        }) = y_A, \quad  \forall \| \bm{\delta} \|_2 < R
    \end{split},
\end{equation}
where  the certified radius depends on the base classifier $f$ and the Gaussian Standard Deviation $\sigma$
\begin{equation}
    R = \frac{\sigma}{2} (\Phi^{-1}(\mathbb{P}(f(\bm{x} + \bm{\xi}) = y_A)) - \Phi^{-1}( \mathbb{P}(f(\bm{x} + \bm{\xi}) = y_B))), 
    \label{eq:rs_robust_guarantee}
\end{equation}
where $\Phi^{-1}$ denotes the inverse of the standard Gaussian CDF. 

\subsection{Latent Space Randomized Smoothing}
To implement LS-RS, we propose to split a neural network into two sub-networks $f_e$ and $f_c$ as shown in Figure~\ref{fig:split}. Sampling the parameterized Gaussian noise in the latent space to augment the latent feature representation $\bm{z}$ requires only the second part of the network $f_c$ to be ``smoothed''. Formally, we define the '`partially'' smoothed classifier $\tilde{g}$ as following: 
\begin{equation}
    \begin{split}
        \tilde{g}(f_e, f_c, \bm{x}, \sigma^2) &= \underset{\hat{y} \in \mathcal{Y}}{\mathrm{argmax}} \mathbb{P}(f_c(f_e(\bm{x}) + \bm{\xi}) = \hat{y}), \\
        & \mathrm{where} \quad \bm{\xi} \sim \mathcal{N}(\bm{0}, \sigma^2 \bm{I}).
    \end{split}
    \label{eq:smoothed_classifier_latent}
\end{equation}
For simplicity, we use $\bm{z}$ to denote unsmoothed representation in the latent space: $f_e(\bm{x}):=\bm{z}$.
Therefore, Equation~\ref{eq:smoothed_classifier_latent} can be re-formulated as 
\begin{equation}
    \begin{split}
        \tilde{g}(f_e, f_c, \bm{x}, \sigma^2) &= \underset{\hat{y} \in \mathcal{Y}}{\mathrm{argmax}} \mathbb{P}(f_c(\bm{z} + \bm{\xi}) = \hat{y})\\
        & :=  g_c(f_c, \bm{z}, \sigma^2) \\
        & \mathrm{where} \quad \bm{\xi} \sim \mathcal{N}(\bm{0}, \sigma^2 \bm{I}).
    \end{split}
    \label{eq:smoothed_classifier_latent_simple}
\end{equation}
Comparing Equation~\ref{eq:smoothed_classifier_latent_simple} with Equation~\ref{eq:smoothed_classifier}, it is straightforward to derive the robustness guarantee:
\begin{equation}
    \begin{split}
        \forall \, \| \bm{\delta_z} \|_2 < R_z,
    \end{split}
    \label{eq:rs_robust_guarantee_latent_1}
\end{equation}
where, 
\begin{equation}
    R_z = \frac{\sigma}{2} (\Phi^{-1}(\mathbb{P}(f_c(\bm{z} + \bm{\xi}) = y_A)) - \Phi^{-1}( \mathbb{P}(f_c(\bm{z} + \bm{\xi}) = y_B))),
    \label{eq:rs_robust_guarantee_latent_2}
\end{equation}
it is true that
\begin{equation}
    g_c(\bm{z} + \bm{\delta_z}) = y_A.
    \label{eq:rs_robust_guarantee_latent_3}
\end{equation}
Therefore, Equation~\ref{eq:rs_robust_guarantee_latent_1} to Equation~\ref{eq:rs_robust_guarantee_latent_3} provide robustness guarantee for $f_c$.

\subsection{Lipschitz Preserving Layers}



The core idea of our approach is to adopt Lipschitz-preserving layers to derive the certification in the input space from the latent space. Here, we introduce the Lipschitz-preserving layers that we will use to build our network, including orthogonal convolutional layers and non-linear activation.

\textbf{Orthogonal convolutional layers.} A circular convolutional layer $\mathsf{conv}$, parameterized by weights $\bm{W} \in \mathbb{R}^{c_{out} \times c_{in} \times n \times n}$, is orthogonal if its input $\bm{X} \in \mathbb{R}^{c_{in} \times n \times n}$ and output $\mathsf{conv}(\bm{X}) \in \mathbb{R}^{c_{out} \times n \times n}$ satisfy that   
\begin{equation}
\begin{split}
    \| \mathsf{conv}(\bm{X}) \|_F &\equiv \|\bm{X}\|_F, \\
    \| \frac{\partial{\mathsf{Loss}}}{\partial{\mathsf{conv}(\bm{X})}} \| &\equiv \| \frac{\partial{\mathsf{Loss}}}{\partial{\bm{X}}} \| , \, \forall \bm{X}.
\end{split}
\end{equation}

Existing works that implement orthogonal convolutional layers are mainly in two categories, namely encouraging orthogonality through a penalization term in the objective function and enforcing orthogonal through parameterization or model design. \cite{gouk2021regularisation, tsuzuku2018lipschitz, li2019preventing, sedghi2018singular, singla2019bounding, anil2019sorting, trockman2021orthogonalizing, jia2019orthogonal, wen2020towards, jia2017improving, cisse2017parseval} The former does not guarantee orthogonality (for instance, computing and regularizing the largest singular value of a convolutional layer) whereas the latter does guarantee. For our purpose of obtaining Lipschitz preserving layers, we use parameterization or model design, i.e., the latter, to enforce orthogonality.

Since the Cayley transform of a skew-symmetric matrix is always orthogonal, \citet{trockman2021orthogonalizing} proposed to apply the Cayley transform to skew-symmetric convolutions in the Fourier domain to parameterize such convolutional layers to be orthogonal. Mathematically, for a skew-symmetric matrix $\bm{A}$, satisfying $\bm{A} = -\bm{A}^T$, the Cayley transform guarantees the matrix $\bm{Q}$ to be orthogonal:
\begin{equation}
    \bm{Q} = (\bm{I} - \bm{A})(\bm{I} + \bm{A})^{-1} = (\bm{I} + \bm{A})^{-1} - \bm{A}(\bm{I} + \bm{A})^{-1}.
    \label{eq:cayley_transform}
\end{equation}
However, directly applying the Cayley transform to the convolutions of the neural network could be problematic: even if convolutions could be easily skew-symmetrized, it is rather inefficient to find their inverse. Technically, one can firstly perform the Fast Fourier Transform ($\mathsf{FFT}$) on the weights of the conlutoinal layer $\bm{W}$ and the input tensor $\bm{X}$, converting them into spectral domain: $\Tilde{\bm{W}}= \mathsf{FFT}(\bm{W})$ and $\Tilde{\bm{X}}= \mathsf{FFT}(\bm{X})$. Correspondingly, in the Fourier domain, the resulted $(i,j)^{th}$ pixel after the convolution and the inverse convolution could be computed using
$$
\mathsf{FFT}(\mathsf{conv}{(\bm{X})})[:,i,j] = \Tilde{\bm{W}}[:,:,i,j] \Tilde{\bm{X}}[:,i,j]
$$ 
and 
$$
\mathsf{FFT}(\mathsf{conv}^{-1}{(\bm{X})})[:,i,j] = \Tilde{\bm{W}}[:,:,i,j]^{-1} \Tilde{\bm{X}}[:,i,j].
$$
Then, the Fourier-domain weights for the skew-symmetric convolution (using the conjugate transpose) and certain matrices required for inverse $\mathsf{FFT}$ are computed:
$$
\Tilde{\bm{A}}[:,:,i,j] := \Tilde{\bm{W}}[:,:,i,j] - \Tilde{\bm{W}}[:,:,i,j]^*
$$
and
$$
\Tilde{\bm{Y}}[:,i,j] := (\bm{I} + \Tilde{\bm{A}}[:,:,i,j])^{-1} \Tilde{\bm{X}}[:,i,j].
$$

Next, based on all matrices computed just now, it is possible to compute the Cayley transform:
$$
\Tilde{\bm{Z}}[:,i,j] := \Tilde{\bm{Y}}[:,i,j] - \Tilde{\bm{A}}[:,:,i,j]\Tilde{\bm{Y}}[:,i,j].
$$
Plugging in $\Tilde{\bm{A}}$ and $\Tilde{\bm{Y}}$, the orthogonal convolution in the spectral domain is achieved according to Equation~\ref{eq:cayley_transform}:
\begin{equation}
    \begin{split}
        \Tilde{\bm{Z}} &= (\bm{I} + \Tilde{\bm{A}})^{-1}\Tilde{\bm{X}} - \Tilde{\bm{A}}(\bm{I} + \Tilde{\bm{A}})^{-1} \Tilde{\bm{X}} \\
        &= [(\bm{I} + \Tilde{\bm{A}})^{-1} - \Tilde{\bm{A}}(\bm{I} + \Tilde{\bm{A}})^{-1}] \Tilde{\bm{X}} \\
        & := \Tilde{\bm{Q}} \Tilde{\bm{X}}, 
    \end{split}
\end{equation}
where $\Tilde{\bm{Q}}$ is orthogonal. Therefore, ultimately, the results in the spatial domain could be obtained by applying inverse $\mathsf{FFT}$ to $\Tilde{\bm{Z}}$:
$$
\mathsf{FFT}^{-1}(\Tilde{\bm{Z}}),
$$
which is the output of the orthogonal convolutional layer.


\textbf{GroupSort activation.} We adopt an
alternative activation function to build our encoder, which is called GroupSort \cite{anil2019sorting}. GroupSort separates
the variables before the activation into groups, sorts each group into ascending order, and outputs the combined vector. Note that GroupSort is both Lipschitz and gradient norm-preserving. The Lipschiz-preserving property of GroupSort enables us to restrain our encoder to be orthogonal. As for preserving the gradient norm, this property contributes to the training of the orthogonal encoder, since there will be no gradient vanishing problem. (Unlike ReLU, which is not norm-preserving and could lead to gradient vanishing.) 


The major advantage of using these modules is that the Lipschitz constant of these layers is 1. Therefore, the concatenation of them, including orthogonal fully connected layers, orthogonal convolutional layers and other Lipschitz-preserving non-linear functions, will lead to a orthogonal network with a global Lipschitz constant 1. This property enables us to establish the relationship between the certified radius computed in the latent space and that of the input space easily. More precisely, the certified radius in the input space will be preserved after the Lipschitz-preserving layers.

\section{Guarantees and Analysis}

\subsection{Robustness Guarantee}
From Equation~\ref{eq:rs_robust_guarantee_latent_1} to Equation~\ref{eq:rs_robust_guarantee_latent_3}, we know how the certified radius could be computed for the latent representation. In other words, we can only guarantee that the perturbations in the latent space within $R_z$ will not cause wrong predictions for the given input example. The next step is to understand the relationship between the perturbation in the latent space and the perturbation in the input space. Obviously, it is not sufficient to find robustness guarantee in the latent space, since the attackers will always try to directly perturb the images in the input space. How can we find the certified radius in the input space when the certified radius in the latent space is available?
\begin{theorem}
\label{theo:latent_certification}
Split a base classifier $f$ into an encoder $f_e$ and a classifier $f_c$. Let $f_c$ be $L$-Lipschitz, then within the certified radius $R_z/L$, it is guaranteed that all possible adversarial examples $\bm{x}'$ will labeled by $\tilde{g}(f_e, f_c, \bm{x}, \sigma^2)$ as $y_A$, which is the same as $\bm{x}$.
\end{theorem}
\begin{proof}
    $L$-Lipschitz of $f_e$ provides
$$
    \|f_e{(\bm{x})} - f_e{(\bm{x}')}\| = \|\bm{z}- \bm{z}'\| \leq L\|\bm{x} - \bm{x}'\|.
    \label{eq:Lipschitz_encoder}
$$
Given certification
$$
    \forall \, \| \bm{\delta_z} \| < R_z, \quad g_c(\bm{z} + \bm{\delta_z}) = y_A.
    \label{eq:robust_guarantee_proof}
$$
Therefore, if $\|\bm{x} - \bm{x}' \| < R_z / L$, then
$$
    \|\bm{\delta}_z \| = \|\bm{z}- \bm{z}'\| \leq L\|\bm{x} - \bm{x}'\| < R_z,
$$
and 
$$
    g_c(f_c, \bm{z}', \sigma^2) = \tilde{g}(f_e, f_c, \bm{x}', \sigma^2) = y_A.   
$$
\end{proof}
Theorem~\ref{theo:latent_certification} demonstrates how to establish the robustness certification in the input space using the certification of latent space, when the encoder is L-Lipschitz. Obviously, since the encoder is allowed to amplify the input signal $L$-times, the certified radius computed in the latent space must be divided by $L$ to obtain the radius in the input space. Moreover, when use norm-preserving layers to build the encoder, indicating that the Lipschitz constant of $f_e$ is $1$, the certified radius in the input space is bounded by the certified radius computed in the latent space.

\subsection{Does Rescaling of Lipschitz Affect?}
When we review the proof of Theorem~\ref{theo:latent_certification}, it is to notice that theoretically, orthogonality is not the necessary condition to derive the robustness guarantee. That is, it is not required to use layers with Lipschitz constant 1 to build the encoder $f_e$. However, we argue that rescaling the Lipschitz constant will not affect the tightness of the bounds.

Consider an $L$-Lipschitz encoder $f_e^{(1)}$. We compute the certified radius in the latent space $R_z$, using Theorem~\ref{theo:latent_certification}. Therefore, the resulting certified radius in the input space is $R_z / L$. For $f_e^{(1)}$, we can always split it into two modules, a multiplication module of constant $L$ and an orthogonal network $f_e^{(2)}$. 
As a result, given two input images $\bm{x}$ and $\bm{x}'$, we have
$$
    \| \bm{f_e}^{(1)}(\bm{x}) - \bm{f_e}^{(1)}(\bm{x}')\| \leq L \| \bm{x} - \bm{x}'\|.
$$
Similarly, for $f_e^{(2)}$, it is true that 
$$
    \| \bm{f_e}^{(2)}(L\bm{x}) - \bm{f_e}^{(2)}(L\bm{x}')\| \leq \| L\bm{x} - L\bm{x}'\| = L \| \bm{x} - \bm{x}'\|.
$$

Therefore, it is sufficiently representative to use orthogonal encoders to compute tight certified radius. Using an encoder with larger Lipschiz constant will not contribute to a better certified radius, since we can just multiply the input with this large constant and then use an orthogonal encoder to get the equivalent result.

\section{Experiments}
\label{sec:exp}
In this section, we evaluate our proposed latent space randomized smoothing on two benchmark datasets. In order to address the inefficiency of randomized smoothing we pointed out in the first section, we compare the efficiency of our proposed LS-RS and baseline IS-RS by reporting average time consumption for certifying one test sample. Moreover, to evaluate the tightness of the robustness guarantee provided by LS-RS we also report the averaged certified radius (ACR) $\frac{1}{N_{test}}\sum_{i}R_i$ \cite{zhai2020macer} of the input space for both datasets. Finally, we present the ablation study of tuning the depth of the latent space to show how LS-RS enables great flexibility in term of controlling the depth of the classifier and the dimensionality of the latent space.

\noindent \textbf{Baselines and experimental settings.} We use Input Space Randomized Smoothing (IS-RS) \cite{cohen2019certified} as the baseline algorithm. For CIFAR10 dataset, we use ResNet18 as the baseline model and modify the architecture by converting its first convolutional layer to be orthogonal and replacing its first $K$ residual blocks with orthogonal residual blocks. As for ImageNet, we use WideResNet34 as the baseline architecture and follow the similar routine to obtain its orthogonal version. We evalute our models using GeForce RTX 2080Ti 11GB.

\noindent \textbf{Implementation of the networks.} We show how the orthogonal version of resnets could be implemented by firstly demonstrating the orthogonal skip connection and then the resulted network with substitutions. As argued by \citet{trockman2021orthogonalizing}, the Lipschitz constant of residual connections $\bm{f}_{res}$ could be ensured by making the main branch $\bm{f}_{main}$ and the skip connection $\bm{X}$ a convex combination with a new learnable parameter $\alpha \in [0,1]$: $f_{res}(\bm{X}) = \alpha f_{main}(\bm{X}) + (1-\alpha)(\bm{X})$. Therefore, we use the modified skip connection (shown in Figure~\ref{fig:rescaled_skip_connection}) to build our orthogonal encoder in all experiments. Moreover, as we mentioned just now, LS-RS provides us with great flexibility of tuning the depth of the latent space. As shown in Figure~\ref{fig:fraction_replacement}, we can control how many the original free skip connections could be replaced by the orthogonal ones. The fraction reveals exactly how we can control the depth as well as the dimensionality of the latent space.

\begin{figure}[!htbp]
     \centering
     \begin{subfigure}[b]{0.45\textwidth}
         \centering
         \includegraphics[width=0.5\textwidth]{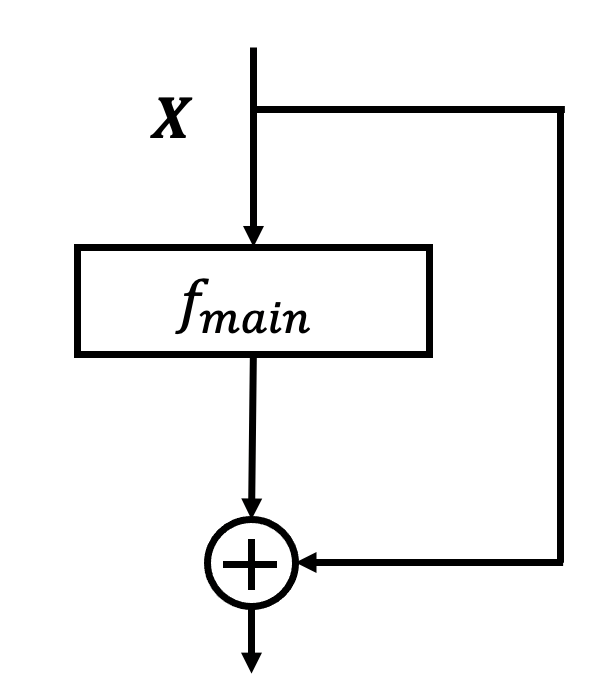}
         \caption{vanilla skip connection}
         \label{fig:vanilla_skip_connection}
     \end{subfigure}
     \begin{subfigure}[b]{0.45\textwidth}
         \centering
         \includegraphics[width=0.5\textwidth]{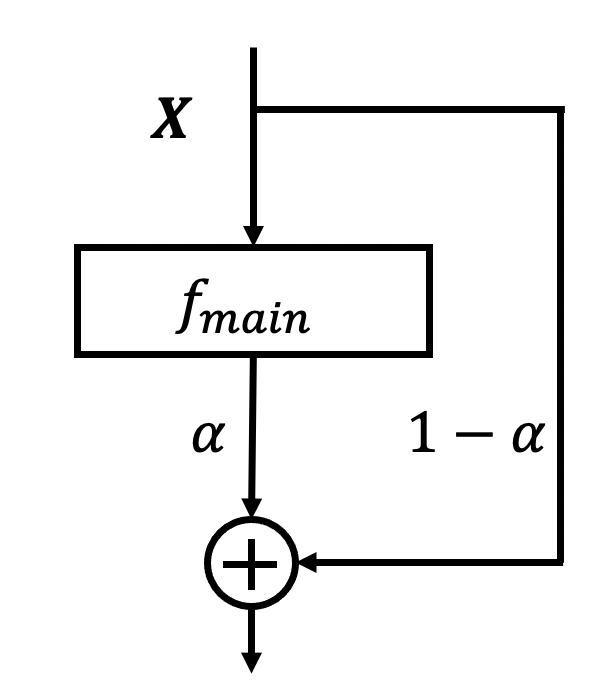}
         \caption{convex combination of two branches}
         \label{fig:rescaled_skip_connection}
     \end{subfigure}
    \caption{Comparison between the original skip connection and the modified skip connection.}
    \label{fig:skip_connections}
\end{figure}
\begin{figure}[!htbp]
     \centering
     \includegraphics[width=0.8\textwidth]{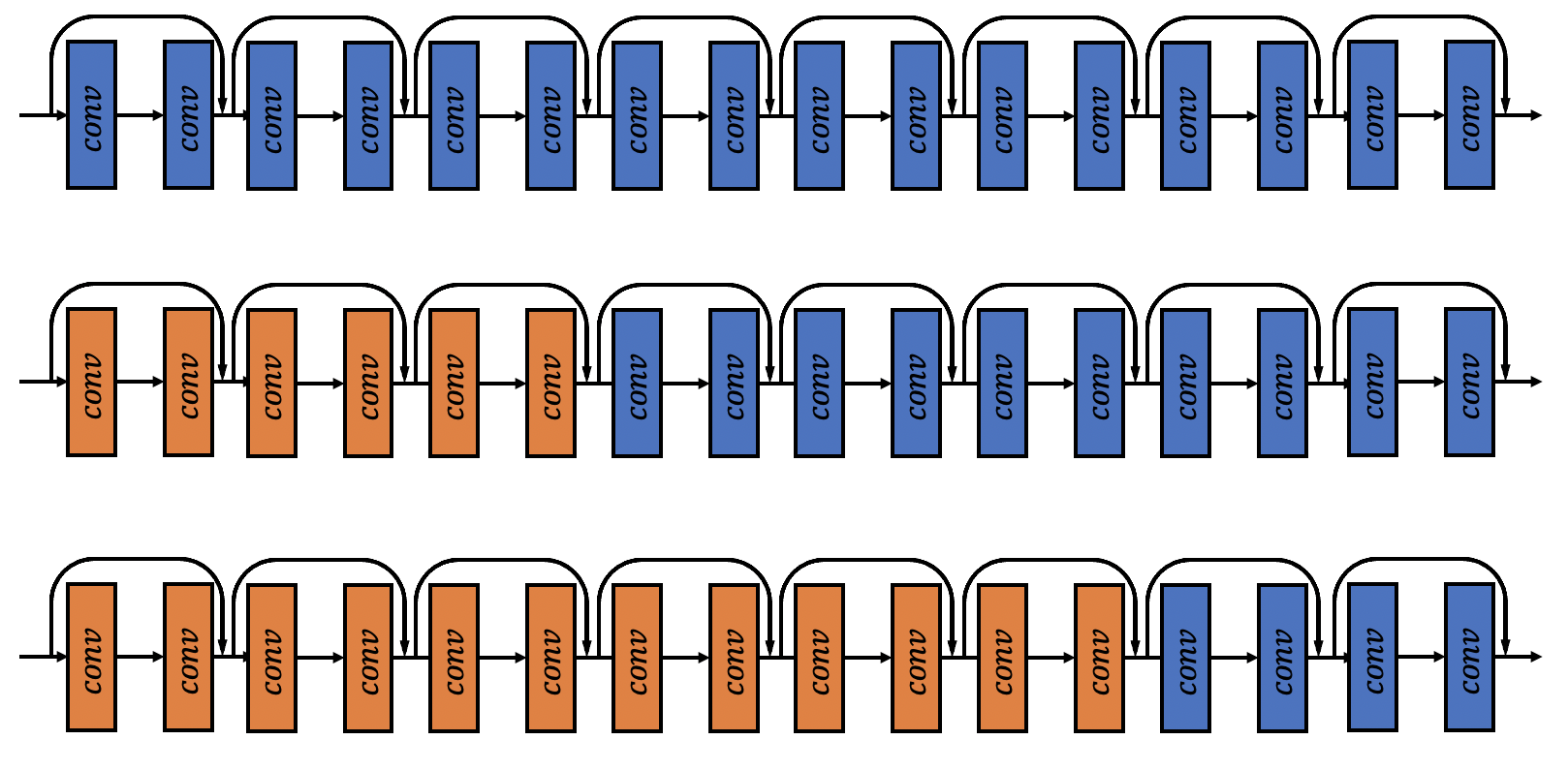}
    \caption{Replace the skip connections in a network with orthogonal skip connections. The first network is the baseline network without any orthogonal module. For the second network, three out of eight skip connections are substituted by orthogonal ones, whereas for the third network, six out of eight are replaced.}
    \label{fig:fraction_replacement}
\end{figure}

\noindent \textbf{Experimental results and analysis.} In this section, by observing the results showing in Table~\ref{tab:CIFAR_compare}, we firstly conclude that the speed of certification of LS-RS is much faster than IS-RS, while achieving comparable robustness guarantee on CIFAR10 dataset. Moreover, we also include the experimental results on ImageNet in Appendix, which also verify this statement. Note that in all tables, the hyperparameter $FoR$ refers to the fraction of replacement with orthogonal skip connections. Speaking of the flexibility provided by LS-RS, we also report Table~\ref{tab:CIFAR_space_tune} to show the effect of the depth of the latent space. Obviously, the deeper the latent space is established, the faster the certification could proceed, whereas both the robustness and the accuracy would be sacrificed. Actually, according to the concentration theory, the number of Gaussian samples needed to approximate the expectation using the empirical mean is higher in a higher-dim space. This intuition is problematic. Since in practice, it is likely that the latent space happens to be a space whose dimensionality is raised in comparison to the input space. In the experiments on CIFAR10, for Table~\ref{tab:CIFAR_compare}, the dimensionality of the input space is $3 \times 32  \times 32 = 3072$, whereas the dimensionality of the latent space is $128 * 28 * 28 = 32768$. As shown in the table, even if we raise the dimensionlity in the latent space and do not sample more Gaussian samples, we are still able to achieve much higher efficiency and comparable certified radius. Finally, concerning the performance drop in Table~\ref{tab:CIFAR_space_tune}, one possible reason is that the training of deep orthogonal layers could be challenging. Unlike training a normal neural network, there is no batch normalization or similar training-stabilizing methods available to optimize the training process of orthogonal networks. Moreover, since the Lipschitz constants of the orthogonal convolutional layers are restricted to be 1, indicating that the expressive power of the encoder $f_e$ could be still limited. With more norm-preserving modules in the network, the expressive power of the entire network could be severely limited, leading to the less satisfying results.

\begin{table}[!htbp]
\centering
\resizebox{0.7\columnwidth}{!}{\begin{tabular}{c|c|c|c|c|c}
\hline
\textbf{Defense} & $\sigma$ & FoR & ACR & Accuracy (\%) &time (s/example)  \\
\hline
IS-RS & 0.00 & 0/18 & 0.000 & 93.15 &  - \\
IS-RS & 0.25 & 0/18 & 0.472 & 80.78 & 11.212 \\
IS-RS & 0.50 & 0/18 & 0.564 & 68.24 & 12.747 \\
IS-RS & 1.00 & 0/18 & 0.532 & 49.35 & 12.953 \\
\hline
\hline
LS-RS & 0.00 & 9/18 & 0.000 & 86.92 & - \\
LS-RS & 0.25 & 9/18 & 0.444 & 77.56 & 7.904 \\
LS-RS & 0.50 & 9/18 & 0.533 & 66.68 & 7.771 \\
LS-RS & 1.00 & 9/18 & 0.524 & 50.72 & 7.985 \\
\hline
\end{tabular}} 
\vspace{1.0em}
\caption{Efficiency and robustness evaluation on CIFAR10.} 
\label{tab:CIFAR_compare}
\end{table}

\begin{table}[!htbp]
\centering
\resizebox{0.7\columnwidth}{!}{\begin{tabular}{c|c|c|c|c|c}
\hline
\textbf{Defense} & $\sigma$ & FoR & ACR & Accuracy (\%) &time (s/example)  \\
\hline
LS-RS & 0.00 & 1/18 & 0.000 & 93.15 & - \\
LS-RS & 0.25 & 1/18 & 0.478 & 80.18 & 12.689 \\
LS-RS & 0.50 & 1/18 & 0.562 & 68.89 & 13.038 \\
LS-RS & 1.00 & 1/18 & 0.522 & 50.55 & 12.879 \\
\hline
\hline
LS-RS & 0.00 & 5/18 & 0.000 & 91.26 & - \\
LS-RS & 0.25 & 5/18 & 0.475 & 79.70 & 10.375 \\
LS-RS & 0.50 & 5/18 & 0.561 & 68.09 & 9.643 \\
LS-RS & 1.00 & 5/18 & 0.540 & 50.07 & 9.913 \\
\hline
\hline
LS-RS & 0.00 & 9/18 & 0.000 & 86.92 & - \\
LS-RS & 0.25 & 9/18 & 0.444 & 77.56 & 7.904 \\
LS-RS & 0.50 & 9/18 & 0.533 & 66.68 & 7.771 \\
LS-RS & 1.00 & 9/18 & 0.524 & 50.72 & 7.985 \\
\hline
\hline
LS-RS & 0.00 & 13/18 & 0.000 & 79.40 & - \\
LS-RS & 0.25 & 13/18 & 0.387 & 71.32 & 6.131 \\
LS-RS & 0.50 & 13/18 & 0.485 & 62.14 & 5.807 \\
LS-RS & 1.00 & 13/18 & 0.480 & 47.70 & 5.827 \\
\hline
\end{tabular}} 
\vspace{1.0em}
\caption{Tuning the depth of the latent space for ResNet18 on CIFAR10.} 
\label{tab:CIFAR_space_tune}
\end{table}

\section{Related Work}
\label{sec:rela}
State-of-the-art adversarial defenses can be categorized into empirical defenses and certifiable defenses. Empirical defenses are robust to known adversarial attacks, but are still vulnerable to unknown stronger attacks. Adversarial training, as one of the most powerful empirical defenses, has been demonstrating its power in defending against strong attacks. By including adversarial examples during training phase, the network can directly learn how to classify adversarial examples ~\citep{madry2017towards, shafahi2019adversarial, zhang2019you} correctly. However, such defense could be subverted by stronger attacks such as iterative attacks~\citep{qian2018l2} or non-uniform attacks~\citep{zeng2020adversarial}. By contrast, certified defenses can provide provable robustness of the models against specific adversarial perturbations. They work by obtaining the perturbation $\delta$ with minimum $\lVert \delta \rVert_p$ such that $f(x) \neq f(x+\delta)$, where $f$ is a classifier and $x$ is the input data~\citep{cheng2017maximum,lomuscio2017approach,dutta2018output,fischetti2017deep}. Since the problem is NP-hard, a relaxation of the non-linearities can be quite useful. Linear inequality constraints show better efficiency~\citep{singh2018fast,gehr2018ai2,zhang2018efficient}. Some defenses integrate the verification methods into the training process, trying to minimize the robust loss directly. A bound derived with a semi-definite programming (SDP) relaxation was minimized as a regularizer~\citep{raghunathan2018certified}. In addition, ~\cite{wong2017provable} presents a similar defense but the upper bound is relaxed with a LP relaxation. Though such certified defense gives provable robustness, they cannot scale to large datasets and arbitrary model types.

\textbf{Randomized smoothing.} In a recent work~\citep{cohen2019certified}, randomized smoothing was formally proposed to provide a tight robustness guarantee for deep neural networks. By adding i.i.d. Gaussian noises to the inputs of a normally trained base classifier, a smoothed classifier is obtained to provide provably robust against $l_2$-norm bounded perturbations with statistical confidence. Based on the baseline randomized smoothing, further following-up works were carried out and proposed approaches to augment the efficacy of randomized smoothing. Substantially boosting certifiable robustness of smoothed classifiers, \citet{salman2019provably} combined adversary training. In \citep{zhai2020macer}, a modified version of randomized smoothing was proposed by injecting the approximated average certified radius into the training objective, with which the ACR of test samples could be further improved. Randomized smoothing is powerful and efficient during training. By simply augmenting the train samples with Gaussian noise, the computational cost is still satisfying. In fact, randomized smoothing is one of few certifiable defenses that could be applied to large datasets, for instance ImageNet. However, during test time, the efficiency of randomized smoothing is frustrating. In order to maintain the concentration bounds and obtain robustness guarantee with high confidence, a large quantity of Gaussian noises must be sampled (usually over 100000 for a single test sample), which could be extremely time consuming.

\textbf{Estimating Lipschitz constant of neural networks.} 
In order to map the certification from the latent space back to the input space, one possible solution is to evaluate the Lipschitz property of the encoder, which has never been an easy task due to the non-convexity of deep neural networks. In most works on estimating the Lipschitz of neural networks, it is recommended to firstly evaluate the Lipschitz constants for each layer and then multiply them together to obtain the overall Lipscthiz constant for the network. In \citep{gouk2021regularisation, tsuzuku2018lipschitz}, it is recommended to approximate the 
operator norm as the bound of Lipschitz constant for different layers using power method. However, in practice, power method is usually implemented with limited iterations, which could not provide tight bounds and therefore, leads to weak robustness guarantees. \citet{sedghi2018singular} proposed to compute the largest singular values of convolutional layers using SVD in the Fourier domain, which is much faster than computing SVD for weight matrices directly. Still, even in the Fourier domain, it could rather computationally expensive to perform SVD for each layer. \citet{singla2019bounding} found upper bounds for the spectral norm of convolution kernels by appropriately reshaping the weight matrix, which is computationally efficient but sacrifices the tightness of the bounds.


\vspace{0.3em}
\textbf{Orthogonal convolutional networks.} To avoid the loose estimated Lipschitz bound, we adopt the orthogonal neural networks, whose Lipschitz bound is imposed by their architectures.
One main challenge in designing these networks is to enforce orthogonality of the convolutional layers.
Early works flatten the higher-order convolutional kernel into a matrix, and enforces orthogonality of the resulted matrix~\citep{jia2017improving,cisse2017parseval}. However, this approach does not lead to the orthogonality of the original convolution.
Projected gradient descent via singular value clipping is proposed in \citet{sedghi2018singular}, which is expensive in practice.
Recent works adopt parameterization-based approaches, either using block convolutions~\citep{li2019preventing} Cayley transform~\citep{trockman2021orthogonalizing}, or spectral factorization~\citep{su2021scaling}.
After reviewing their flexibility and scalability, we adopt the algorithm and code of \citep{su2021scaling} to build our orthogonal modules.


\section{Conclusion}
\label{sec:conclu}
Aiming at increasing the efficiency of randomized smoothing, this work studies the potential of performing smoothing in the latent space. We propose Latent Space Randomized Smoothing, which is achieved by sampling Gaussian noises in the latent space. To this end, we are able to find robustness guarantee for the partial network after the latent space. In order to establish the equivalency between the robust certification between the latent space and input space, we adopt orthogonal convolutional layer and the norm-preserving GroupSort activation to build the encoder, the first sub-network. With Lipschitz constant equal to 1, we proved in the paper that the certified radius computed in the latent space is identical to that in the input space and rescaling the Lipschtiz constant of the encoder will not contribute to be tighter bounds. 

Our method LS-RS improves the efficiency of randomized smoothing drastically compared to IS-RS, while achieving slightly worse but still comparable performance on robustness guarantee. The utilization of norm-preserving layers motivates us to analyze adversarial perturbation from a different perspective in the latent space. The core benefit of such approach is that the computational complexity used for sampling and forward pass could be thoroughly reduced. 


\section{Acknowledgement}
\label{sec:ack}
This work is supported by National Science Foundation IIS-1850220 CRII Award 030742-00001 and DOD-DARPA-Defense Advanced Research Projects Agency Guaranteeing AI Robustness against Deception (GARD), and Adobe, Capital One and JP Morgan faculty fellowships.

\newpage
\bibliography{reference}

\begin{thebibliography}{36}
\providecommand{\natexlab}[1]{#1}
\providecommand{\url}[1]{\texttt{#1}}
\expandafter\ifx\csname urlstyle\endcsname\relax
  \providecommand{\doi}[1]{doi: #1}\else
  \providecommand{\doi}{doi: \begingroup \urlstyle{rm}\Url}\fi

\bibitem[Anil et~al.(2019)Anil, Lucas, and Grosse]{anil2019sorting}
C.~Anil, J.~Lucas, and R.~Grosse.
\newblock Sorting out lipschitz function approximation.
\newblock In \emph{International Conference on Machine Learning}, pages
  291--301. PMLR, 2019.

\bibitem[Carlini and
  Wagner(2017)]{AdversarialExamplesAreNotEasilyDetected:BypassingTenDetectionMethods}
N.~Carlini and D.~Wagner.
\newblock Adversarial examples are not easily detected: Bypassing ten detection
  methods.
\newblock 2017.

\bibitem[Cheng et~al.(2017)Cheng, N{\"u}hrenberg, and Ruess]{cheng2017maximum}
C.-H. Cheng, G.~N{\"u}hrenberg, and H.~Ruess.
\newblock Maximum resilience of artificial neural networks.
\newblock In \emph{International Symposium on Automated Technology for
  Verification and Analysis}, pages 251--268. Springer, 2017.

\bibitem[Cisse et~al.(2017)Cisse, Bojanowski, Grave, Dauphin, and
  Usunier]{cisse2017parseval}
M.~Cisse, P.~Bojanowski, E.~Grave, Y.~Dauphin, and N.~Usunier.
\newblock Parseval networks: Improving robustness to adversarial examples.
\newblock In \emph{International Conference on Machine Learning}, pages
  854--863. PMLR, 2017.

\bibitem[Cohen et~al.(2019)Cohen, Rosenfeld, and Kolter]{cohen2019certified}
J.~M. Cohen, E.~Rosenfeld, and J.~Z. Kolter.
\newblock Certified adversarial robustness via randomized smoothing.
\newblock \emph{arXiv preprint arXiv:1902.02918}, 2019.

\bibitem[Dutta et~al.(2018)Dutta, Jha, Sankaranarayanan, and
  Tiwari]{dutta2018output}
S.~Dutta, S.~Jha, S.~Sankaranarayanan, and A.~Tiwari.
\newblock Output range analysis for deep feedforward neural networks.
\newblock In \emph{NASA Formal Methods Symposium}, pages 121--138. Springer,
  2018.

\bibitem[Fischetti and Jo(2017)]{fischetti2017deep}
M.~Fischetti and J.~Jo.
\newblock Deep neural networks as 0-1 mixed integer linear programs: A
  feasibility study.
\newblock \emph{arXiv preprint arXiv:1712.06174}, 2017.

\bibitem[Gehr et~al.(2018)Gehr, Mirman, Drachsler-Cohen, Tsankov, Chaudhuri,
  and Vechev]{gehr2018ai2}
T.~Gehr, M.~Mirman, D.~Drachsler-Cohen, P.~Tsankov, S.~Chaudhuri, and
  M.~Vechev.
\newblock Ai2: Safety and robustness certification of neural networks with
  abstract interpretation.
\newblock In \emph{2018 IEEE Symposium on Security and Privacy (SP)}, pages
  3--18. IEEE, 2018.

\bibitem[Gouk et~al.(2021)Gouk, Frank, Pfahringer, and
  Cree]{gouk2021regularisation}
H.~Gouk, E.~Frank, B.~Pfahringer, and M.~J. Cree.
\newblock Regularisation of neural networks by enforcing lipschitz continuity.
\newblock \emph{Machine Learning}, 110\penalty0 (2):\penalty0 393--416, 2021.

\bibitem[Jia et~al.(2017)Jia, Tao, Gao, and Xu]{jia2017improving}
K.~Jia, D.~Tao, S.~Gao, and X.~Xu.
\newblock Improving training of deep neural networks via singular value
  bounding.
\newblock In \emph{Proceedings of the IEEE Conference on Computer Vision and
  Pattern Recognition}, pages 4344--4352, 2017.

\bibitem[Jia et~al.(2019)Jia, Li, Wen, Liu, and Tao]{jia2019orthogonal}
K.~Jia, S.~Li, Y.~Wen, T.~Liu, and D.~Tao.
\newblock Orthogonal deep neural networks.
\newblock \emph{IEEE transactions on pattern analysis and machine
  intelligence}, October 2019.
\newblock ISSN 0162-8828.
\newblock \doi{10.1109/tpami.2019.2948352}.
\newblock URL \url{https://doi.org/10.1109/TPAMI.2019.2948352}.

\bibitem[Kannan et~al.(2018)Kannan, Kurakin, and
  Goodfellow]{Kannan2018AdversarialLP}
H.~Kannan, A.~Kurakin, and I.~J. Goodfellow.
\newblock Adversarial logit pairing.
\newblock \emph{ArXiv}, abs/1803.06373, 2018.

\bibitem[Krizhevsky et~al.(2009)Krizhevsky, Hinton,
  et~al.]{krizhevsky2009learning}
A.~Krizhevsky, G.~Hinton, et~al.
\newblock Learning multiple layers of features from tiny images.
\newblock 2009.

\bibitem[Kurakin et~al.(2016)Kurakin, Goodfellow, and
  Bengio]{kurakin2016adversarial}
A.~Kurakin, I.~Goodfellow, and S.~Bengio.
\newblock Adversarial machine learning at scale.
\newblock \emph{arXiv preprint arXiv:1611.01236}, 2016.

\bibitem[Li et~al.(2019)Li, Haque, Anil, Lucas, Grosse, and
  Jacobsen]{li2019preventing}
Q.~Li, S.~Haque, C.~Anil, J.~Lucas, R.~B. Grosse, and J.-H. Jacobsen.
\newblock Preventing gradient attenuation in lipschitz constrained
  convolutional networks.
\newblock In \emph{Advances in neural information processing systems}, pages
  15390--15402, 2019.

\bibitem[Lomuscio and Maganti(2017)]{lomuscio2017approach}
A.~Lomuscio and L.~Maganti.
\newblock An approach to reachability analysis for feed-forward relu neural
  networks.
\newblock \emph{arXiv preprint arXiv:1706.07351}, 2017.

\bibitem[Madry et~al.(2017)Madry, Makelov, Schmidt, Tsipras, and
  Vladu]{madry2017towards}
A.~Madry, A.~Makelov, L.~Schmidt, D.~Tsipras, and A.~Vladu.
\newblock Towards deep learning models resistant to adversarial attacks.
\newblock \emph{arXiv preprint arXiv:1706.06083}, 2017.

\bibitem[Qian and Wegman(2018)]{qian2018l2}
H.~Qian and M.~N. Wegman.
\newblock L2-nonexpansive neural networks.
\newblock \emph{arXiv preprint arXiv:1802.07896}, 2018.

\bibitem[Raghunathan et~al.(2018)Raghunathan, Steinhardt, and
  Liang]{raghunathan2018certified}
A.~Raghunathan, J.~Steinhardt, and P.~Liang.
\newblock Certified defenses against adversarial examples.
\newblock \emph{arXiv preprint arXiv:1801.09344}, 2018.

\bibitem[Russakovsky et~al.(2015)Russakovsky, Deng, Su, Krause, Satheesh, Ma,
  Huang, Karpathy, Khosla, Bernstein, et~al.]{russakovsky2015imagenet}
O.~Russakovsky, J.~Deng, H.~Su, J.~Krause, S.~Satheesh, S.~Ma, Z.~Huang,
  A.~Karpathy, A.~Khosla, M.~Bernstein, et~al.
\newblock Imagenet large scale visual recognition challenge.
\newblock \emph{International journal of computer vision}, 115\penalty0
  (3):\penalty0 211--252, 2015.

\bibitem[Salman et~al.(2019)Salman, Yang, Li, Zhang, Zhang, Razenshteyn, and
  Bubeck]{salman2019provably}
H.~Salman, G.~Yang, J.~Li, P.~Zhang, H.~Zhang, I.~Razenshteyn, and S.~Bubeck.
\newblock Provably robust deep learning via adversarially trained smoothed
  classifiers.
\newblock \emph{arXiv preprint arXiv:1906.04584}, 2019.

\bibitem[Sedghi et~al.(2018)Sedghi, Gupta, and Long]{sedghi2018singular}
H.~Sedghi, V.~Gupta, and P.~M. Long.
\newblock The singular values of convolutional layers.
\newblock \emph{arXiv preprint arXiv:1805.10408}, 2018.

\bibitem[Shafahi et~al.(2019)Shafahi, Najibi, Ghiasi, Xu, Dickerson, Studer,
  Davis, Taylor, and Goldstein]{shafahi2019adversarial}
A.~Shafahi, M.~Najibi, A.~Ghiasi, Z.~Xu, J.~Dickerson, C.~Studer, L.~S. Davis,
  G.~Taylor, and T.~Goldstein.
\newblock Adversarial training for free!
\newblock \emph{arXiv preprint arXiv:1904.12843}, 2019.

\bibitem[Shaham et~al.(2018)Shaham, Yamada, and
  Negahban]{shaham2018understanding}
U.~Shaham, Y.~Yamada, and S.~Negahban.
\newblock Understanding adversarial training: Increasing local stability of
  supervised models through robust optimization.
\newblock \emph{Neurocomputing}, 307:\penalty0 195--204, 2018.

\bibitem[Singh et~al.(2018)Singh, Gehr, Mirman, P{\"u}schel, and
  Vechev]{singh2018fast}
G.~Singh, T.~Gehr, M.~Mirman, M.~P{\"u}schel, and M.~Vechev.
\newblock Fast and effective robustness certification.
\newblock In \emph{Advances in Neural Information Processing Systems}, pages
  10802--10813, 2018.

\bibitem[Singla and Feizi(2019)]{singla2019bounding}
S.~Singla and S.~Feizi.
\newblock Bounding singular values of convolution layers.
\newblock \emph{arXiv preprint arXiv:1911.10258}, 2019.

\bibitem[Su et~al.(2021)Su, Byeon, and Huang]{su2021scaling}
J.~Su, W.~Byeon, and F.~Huang.
\newblock Scaling-up diverse orthogonal convolutional networks with a
  paraunitary framework.
\newblock \emph{arXiv preprint arXiv:2106.09121}, 2021.

\bibitem[Szegedy et~al.(2013)Szegedy, Zaremba, Sutskever, Bruna, Erhan,
  Goodfellow, and Fergus]{szegedy2013intriguing}
C.~Szegedy, W.~Zaremba, I.~Sutskever, J.~Bruna, D.~Erhan, I.~Goodfellow, and
  R.~Fergus.
\newblock Intriguing properties of neural networks.
\newblock \emph{arXiv preprint arXiv:1312.6199}, 2013.

\bibitem[Trockman and Kolter(2021)]{trockman2021orthogonalizing}
A.~Trockman and J.~Z. Kolter.
\newblock Orthogonalizing convolutional layers with the cayley transform.
\newblock \emph{arXiv preprint arXiv:2104.07167}, 2021.

\bibitem[Tsuzuku et~al.(2018)Tsuzuku, Sato, and Sugiyama]{tsuzuku2018lipschitz}
Y.~Tsuzuku, I.~Sato, and M.~Sugiyama.
\newblock Lipschitz-margin training: Scalable certification of perturbation
  invariance for deep neural networks.
\newblock \emph{arXiv preprint arXiv:1802.04034}, 2018.

\bibitem[Wen et~al.(2020)Wen, Li, and Jia]{wen2020towards}
Y.~Wen, S.~Li, and K.~Jia.
\newblock Towards understanding the regularization of adversarial robustness on
  neural networks.
\newblock In \emph{International Conference on Machine Learning}, pages
  10225--10235. PMLR, 2020.

\bibitem[Wong and Kolter(2017)]{wong2017provable}
E.~Wong and J.~Z. Kolter.
\newblock Provable defenses against adversarial examples via the convex outer
  adversarial polytope.
\newblock \emph{arXiv preprint arXiv:1711.00851}, 2017.

\bibitem[Zeng et~al.(2020)Zeng, Zhu, Goldstein, and Huang]{zeng2020adversarial}
H.~Zeng, C.~Zhu, T.~Goldstein, and F.~Huang.
\newblock Are adversarial examples created equal? a learnable weighted minimax
  risk for robustness under non-uniform attacks.
\newblock \emph{arXiv preprint arXiv:2010.12989}, 2020.

\bibitem[Zhai et~al.(2020)Zhai, Dan, He, Zhang, Gong, Ravikumar, Hsieh, and
  Wang]{zhai2020macer}
R.~Zhai, C.~Dan, D.~He, H.~Zhang, B.~Gong, P.~Ravikumar, C.-J. Hsieh, and
  L.~Wang.
\newblock Macer: Attack-free and scalable robust training via maximizing
  certified radius.
\newblock \emph{arXiv preprint arXiv:2001.02378}, 2020.

\bibitem[Zhang et~al.(2019)Zhang, Zhang, Lu, Zhu, and Dong]{zhang2019you}
D.~Zhang, T.~Zhang, Y.~Lu, Z.~Zhu, and B.~Dong.
\newblock You only propagate once: Painless adversarial training using maximal
  principle.
\newblock \emph{arXiv preprint arXiv:1905.00877}, 2019.

\bibitem[Zhang et~al.(2018)Zhang, Weng, Chen, Hsieh, and
  Daniel]{zhang2018efficient}
H.~Zhang, T.-W. Weng, P.-Y. Chen, C.-J. Hsieh, and L.~Daniel.
\newblock Efficient neural network robustness certification with general
  activation functions.
\newblock In \emph{Advances in Neural Information Processing Systems}, pages
  4939--4948, 2018.

\end{thebibliography}

\newpage

\clearpage
\appendix

{\begin{center}{\bf \Large Appendix: \mytitle}\end{center}}

\section{Experiments on ImageNet}
\label{app:exp_imagenet}
In this section, we provide the experimental results on ImageNet. As described in the main text, we use WideResNet34 as the baseline neural network architecture and modify its first two layer and the following consecutive 7 residual blocks to be orthogonal. The widen factor for both baseline model and orthogonal model is set to two.

\begin{table}[!htbp]
\centering
\resizebox{0.7\columnwidth}{!}{\begin{tabular}{c|c|c|c|c|c}
\hline
\textbf{Defense} & $\sigma$ & FoR & ACR & Accuracy (\%) &time (s/example)  \\
\hline
IS-RS & 0.00 & 0/34 & 0.000 & 74.70 &  - \\
IS-RS & 0.25 & 0/34 & 0.654 & 73.89 &  99.775 \\
IS-RS & 0.50 & 0/34 & 1.166 & 71.92 &  87.738 \\
IS-RS & 1.00 & 0/34 & 2.123 & 69.19 &  87.893 \\
\hline
\hline
LS-RS & 0.00 & 16/34 & 0.000 & 71.62 &  - \\
LS-RS & 0.25 & 16/34 & 0.532 & 69.69 & 47.909 \\
LS-RS & 0.50 & 16/34 & 1.063 & 68.45 & 48.098 \\
LS-RS & 1.00 & 16/34 & 1.854 & 65.07 & 46.912 \\
\hline
\end{tabular}} 
\vspace{1.0em}
\caption{Efficiency and robustness evaluation on ImageNet.} 
\label{tab:ImageNet_compare}
\end{table}

By observing Table~\ref{tab:ImageNet_compare}, we can conclude that on ImageNet dataset, the average time used for certifying one single image could be reduced from around 91 seconds to around 47 seconds, a 48.3\% efficiency improvement, with 14.07\% degradation in performance for $\sigma = 0.25$, 8.83\% for $\sigma = 0.50$ and 12.67\% for $\sigma = 1.0$. 

We point out that, most current works \cite{li2019preventing, trockman2021orthogonalizing}, adopting orthogonal convolutions to build Lipschitz-bounded neural networks, mainly focus on CIFAR10, and barely scale to ImageNet. It is the first time that orthogonal convolutions are applied on ImageNet dataset. As shown by the results in Table~\ref{tab:ImageNet_compare}, there is a significant improvement in efficiency while a slight performance drop compared to the IS-RS.

\section{Training Details}
\label{app:train_details}
\textbf{Experiments on CIFAR10.} In all experiments, we used identical training hyperparameters except $\sigma$ to train the baseline models as well as our orthogonal models. The models were trained for 100 epochs. The learning rate was initialized to be 0.01 and was adjusted with a decaying factor 0.1 every 30 epochs. The optimizer was momentum with decaying factor 0.9. While evaluating the robustness of the smoothed classifier, all test examples were used.

\textbf{Experiments on ImageNet.} In all experiments, we used identical training hyperparameters except $\sigma$ to train the baseline models as well as our orthogonal models. The models were trained for 90 epochs. The learning rate was initialized to be 0.1 and was adjusted with a decaying factor 0.1 every 30 epochs. The optimizer was momentum with decaying factor 0.9. During test time, all examples of the validation set were used to compute the accuracy. However, for the sake of simplicity, a subset of images were sampled to compute the certified radius. Specifically speaking, we randomly sample one image for each class and evaluate their robustness certification. 

\section{Ethics Statement}
\label{app:Ethics_statement}

{\bf Adversarial examples could raise extremely high threats to modern machine learning systems.} Therefore, it is crucial to develop adversarial defenses.
Deep networks have shown impressive performance on various tasks such as object detection, speech recognition, and game playing. However, they could still fail catastrophically in the presence of small adversarial perturbations, which are imperceptible. The existence of such adversarial examples exposes a severe vulnerability in current ML systems. Therefore, it is vital to develop reliable and efficient defense mechanisms to increase the robustness of such machine learning models in the context of adversarial attacks.

There are two streams for design defense algorithms. Empirical defense mechanisms can defend against existing attacks by including adversarial examples during training but fail when stronger attackers strike. In contrast, defenses with guarantees are much more reliable and could withstand arbitrary attacks. Randomized smoothing (RS) is one of the most potent certifiable defense algorithms. In the framework of RS, the robustness certification for the victim model is obtained by sampling Gaussian noises and using them to augment the test images.

{\bf Randomized smoothing suffers from sampling an enormous number of noises. Our work can reduce the computational complexity and increase the efficiency of RS.}
However, despite the impressive defending power of randomized smoothing, they suffer from sampling an enormous Gaussian noise during the evaluation phase, which could be extremely time-consuming. Our work provides a new methodology of establishing robustness certification for randomized smoothing in a more controllable latent space of neural networks by utilizing orthogonal convolutions. Sampling Gaussian noise in the latent space can save much computational complexity and increase the efficiency of RS. 

Moreover, our work also carries out a new perspective of analyzing the adversarial signals in the latent space within a network, inspiring further potentials to design new defense mechanisms against adversarial attacks.

\end{document}